%% file: main.tex
\documentclass[letterpaper]{article} %
\usepackage{aaai25}  %
\usepackage{times}  %
\usepackage{helvet}  %
\usepackage{courier}  %
\usepackage[hyphens]{url}  %
\usepackage{graphicx} %
\usepackage{natbib}  %
\usepackage{caption} %
\usepackage{algorithm}
\usepackage{algorithmic}
\usepackage{amsmath,amsfonts,amssymb,amsthm}
\usepackage{xspace}
\usepackage{subcaption}
\usepackage[obeyFinal]{todonotes}
\usepackage{multirow}

\usepackage{newfloat}
\usepackage{listings}
\DeclareCaptionStyle{ruled}{labelfont=normalfont,labelsep=colon,strut=off} %
\floatstyle{ruled}
\newfloat{listing}{tb}{lst}{}
\floatname{listing}{Listing}
\usepackage{color}
\usepackage{booktabs}

\title{Enumerating Minimal Unsatisfiable Cores of LTL$_f$ formulas}
\author {
    Antonio Ielo\textsuperscript{\rm 1},
    Giuseppe Mazzotta\textsuperscript{\rm 1},
    Rafael Pe\~naloza\textsuperscript{\rm 2},
    Francesco Ricca\textsuperscript{\rm 1}
}
\affiliations {
    \textsuperscript{\rm 1}University of Calabria, Italy\\
    \textsuperscript{\rm 2}University of Milano-Bicocca, Italy\\
    antonio.ielo@unical.it, giuseppe.mazzotta@unical.it, rafael.penaloza@unimib.it, francesco.ricca@unical.it
}

\newtheorem{definition}{Definition}
\newtheorem{lemma}[definition]{Lemma}

\newtheorem{theorem}[definition]{Theorem}

\newtheorem{example}[definition]{Example}

\input{macros}
\begin{document}

\maketitle

\begin{abstract}
Linear Temporal Logic over finite traces (\LTLf) is a widely used formalism with applications in AI, process mining, model checking, and more. 
The primary reasoning task for LTLf is satisfiability checking; yet, the recent focus on explainable AI has increased interest in analyzing inconsistent formulas, making the enumeration of minimal explanations for infeasibility a relevant task also for \LTLf. 
This paper introduces a novel technique for enumerating minimal unsatisfiable cores (MUCs) of an \LTLf specification. 
The main idea is to encode a \LTLf formula into an Answer Set Programming (ASP) specification, such that the minimal unsatisfiable subsets (MUSes) of the ASP program directly correspond to the MUCs of the original \LTLf specification.
Leveraging recent advancements in ASP solving yields a MUC enumerator achieving good performance in experiments conducted on established benchmarks from the literature.
\end{abstract}

\section{Introduction}
\input{section/introduction}

\section{Related Work}
\input{section/related_works}

\section{Preliminaries}
We briefly recap required notions of Linear Temporal Logic over Finite Traces (\LTLf)~\cite{DBLP:conf/ijcai/GiacomoV13} and Answer Set Programming (ASP)~\cite{DBLP:journals/cacm/BrewkaET11}.

\subsection{Answer Set Programming}
\input{section/preliminaries_asp}

\subsection{Linear Temporal Logic over Finite Traces}
\input{section/preliminaries_ltlf}

\section{Method}

Technique approach proposed in this paper relies on leveraging ASP \emph{minimal unsatisfiable sets} (MUSes) enumeration algorithms to generate a sequence of candidate \emph{minimal unsatisfible cores} (MUCs) for an \LTLf formula. In order to put in place such approach, a formal connection between these objects must be established. In this section, we introduce the notion of \emph{probe} and \emph{$k$-MUC} to investigate this relationship. A probe is an abstraction over the class of logic programs with suitable properties to apply the approach herein presented; $k$-MUCs are a \textit{relaxation} wrt model length of the concept of MUC, which reveals to be more suitable for ASP-based reasoning.

\input{section/method}

\section{A concrete example of probe}\label{sec:theprobe}
\input{section/probe}

\section{Experiments}
\input{section/experiments}

\section{Conclusions}
Satisfiability of temporal specifications expressed in \LTLf play an important role in several artificial intelligence application domains~\cite{DBLP:journals/amai/BacchusK98,DBLP:conf/kr/CalvaneseGV02,DBLP:conf/aips/GiacomoMMS16,DBLP:conf/ecp/GiacomoV99,DBLP:conf/aips/GiacomoIFP19}. Therefore, in case of unsatisfiable specifications, detecting reasons for unsatisfiability --- e.g., computing its minimal unsatisfiable cores --- is of particular interest. This is especially true whenever the specification under analysis \emph{is expected to be satisfiable}. 

Recent works~\cite{10.1093/logcom/exad049,DBLP:journals/jair/RoveriCFG24} propose several approaches for single MUC computation, but do not investigate enumeration techniques for MUCs. 

However, enumerating MUCs for \LTLf specifications is pivotal to enabling several reasoning services, such as some explainability tasks~\cite{DBLP:journals/ai/Miller19}, as it is the case for propositional logic~\cite{DBLP:conf/ismvl/Silva10,DBLP:journals/ai/Marques-SilvaJM17}.

In this paper, we propose an approach for characterizing MUCs of \LTLf formulae as minimal unsatisfiable subprograms (MUS) of suitable logic programs, introducing the notion of probe. 
This enables to implement \LTLf MUC enumeration techniques by exploiting off-the-shelf ASP and \LTLf reasoners, similarly to SAT-based domain agnostic MUC enumeration techniques à la~\cite{DBLP:conf/lpar/BendikC18}. 

The approach presented herein is modular with respect to ASP \& \LTLf reasoners, which essentially constitute two sub-modules of the system, and with respect to the logic program that is used to extract MUCs via its MUSes.

We implement this strategy in \texttt{mus2muc}, using the ASP solver \texttt{wasp} and the \LTLf solver \texttt{aaltaf}. Our experiments show \texttt{mus2muc} is effective at enumerating MUCs of unsatisfiable formulae that are commonly used in \LTLf literature as benchmarks, as well as being competitive with available state-of-the-art for single MUC computation.

To the best of our knowledge, this represent the first attempt to address this task in the \LTLf setting. 

As far as future works are concerned, we are interested in studying how the choice of probes affect MUCs computation in our setting, as well as providing ad-hoc implementations for closely related \LTLf tasks, such explaining and repairing incosistent Declare specification in the realm of process mining.

\bibliography{refs}
\end{document}

%% file: macros.tex
\newcommand{\ttt}[1]{\texttt{#1}}

\newcommand{\Next}{\textsf{X}\ }

\newcommand{\Until}{\  \textsf{U}\ }

\newcommand{\LTLf}{LTL$_\text{f}$\xspace}

\newcommand{\MUS}{\textsf{MUS}}
\newcommand{\MSM}{\textsf{MSM}}
\newcommand{\MUC}{\textsf{MUC}}
\newcommand{\Formula}{\textsf{Formula}}
\newcommand{\ObjAtoms}{\textsf{Atoms}}
\newcommand{\naf}{\mathop{not}}

%% file: section/introduction.tex
Linear temporal logic over finite traces (\LTLf) \cite{DBLP:conf/ijcai/GiacomoV13} 
is a simple, yet powerful language for expressing and reasoning
about temporal specifications, that is known to be particularly well-suited for applications in Artificial Intelligence (AI)~\cite{DBLP:journals/amai/BacchusK98,DBLP:conf/kr/CalvaneseGV02,DBLP:conf/aips/GiacomoMMS16,DBLP:conf/ecp/GiacomoV99}.

Perhaps its most widely recognised
use to-date is as the logic underlying temporal process modelling 
languages such as Declare \cite{PeSV07}.
Very briefly, a Declare specification is a set of contraints on
the potential evolution of a process, which is expressed through
a syntactic variant of a subclass of \LTLf formulas. The full
specification can thus be seen as a conjunction of \LTLf formulas. 
As specifications become bigger---specially when they are automatically
mined from trace logs \cite{CiMo22}, it is not uncommon
to encounter inconsistencies (i.e., business process models which
are intrinsically contradictory) or other errors. 

To understand and correct these errors, it is thus important to 
highlight the sets of formulas in the specification that are 
responsible for them~\cite{10.1093/logcom/exad049,DBLP:journals/jair/RoveriCFG24}. Specifically, we are interested in computing
the \emph{minimal unsatisfiable cores} (MUCs): subset-minimal subsets 
of formulas (from the original specification) that are collectively
inconsistent~\cite{DBLP:journals/constraints/LiffitonPMM16,10.1093/logcom/exad049,DBLP:journals/jair/RoveriCFG24}. 
These can be seen as the prime causes of the error. 
Notably, a single specification can yield multiple MUCs of varying sizes, depending on the specific constraints involved. 
Exploring more than one MUC can be crucial for analyzing and understanding the causes of incoherence (as recognized in explainable AI~\cite{DBLP:journals/ai/Miller19,DBLP:conf/kr/AudemardKM20}). Thus, a system capable of efficiently enumerating MUCs would be of significant value.

A similar problem has been studied in the field of answer set 
programming (ASP) \cite{DBLP:journals/cacm/BrewkaET11,DBLP:journals/ngc/GelfondL91}, where the
goal is to find \emph{minimal unsatisfiable subsets} (MUS) of 
atoms that make an ASP program incoherent~\cite{DBLP:journals/ai/BrewkaTU19,DBLP:conf/sat/Mencia020,DBLP:journals/ai/AlvianoDFPR23}. 
In recent years, efficient implementations of MUS enumerators have been presented \cite{DBLP:journals/ai/AlvianoDFPR23}. 

Our goal in this paper is to take advantage of both the declarativity of tha ASP language and the efficiency of ASP systems to also enumerate MUCs of \LTLf formulas. Hence, we present a new transformation which constructs, given a set of \LTLf formulas, an ASP program whose MUSes are in a biunivocal corrispondence with the MUCs of the original specification. 
Importantly, although we base our reduction 
on a well-known 
encoding of \LTLf bounded satisfiability~\cite{DBLP:journals/jair/FiondaG18,DBLP:conf/lpnmr/FiondaIR24}, 
the idea is general enough to be applicable to other decision procedures, as long as it can be expressed in ASP. 
To improve its efficiency, our enumerator checks for unsatisfiability iteratively by considering traces of increasing length based on a progression strategy~\cite{DBLP:journals/constraints/MorgadoHLPM13}.
To the best of our knowledge, we provide the first MUC enumerator
for \LTLf. 

We empirically compared our implementation with existing systems designed to produce only \emph{one} MUC (or just one potentially non-minimal unsatisfiable core)~\cite{10.1093/logcom/exad049,DBLP:journals/jair/RoveriCFG24} and observed that our system---despite being more general---is competitive against those on established benchmarks from the literature. 
Importantly, MUC enumeration is very efficient as well.

%% file: section/related_works.tex
The task of computing MUCs has been considered, under different names,
for several representation languages including propositional logic
\cite{LiSa08}, constraint satisfaction problems \cite{MenciaM14}, databases \cite{MeliouRS14}, description
logics \cite{SchlobachC03}, and ASP \cite{DBLP:journals/ai/AlvianoDFPR23} among many others.
For a general overview of the
task and known approaches to solve it, see \cite{Pena20}.

Although the task was briefly studied for LTL (over infinite traces)
in \cite{BaPe-JAR10}, it was only recently considered for the 
specific case of \LTLf \cite{10.1093/logcom/exad049,DBLP:journals/jair/RoveriCFG24}.
Interestingly, for \LTLf the focus has been only on computing one
(potentially non-minimal) unsatisfiable core. To our knowledge,
we are the first to propose a full-fletched \LTLf MUC enumerator.

The idea of using a highly optimised reasoner from one language to 
enumerate MUCs from another one was already considered, first 
exploiting SAT solvers \cite{SeVe09} and later on using ASP solvers 
\cite{PeRi22}. Our approach falls into the latter class.
Our reduction to ASP is inspired on the automata-based satisfiability
procedure, previously used for SAT-based satisfiability checking
\cite{DBLP:journals/jair/FiondaG18,LiPZVR20}, alongside an incremental
approach that verifies the (non-)existence of models up to a certain
length.

%% file: section/preliminaries_asp.tex
\paragraph{Syntax and semantics.}
A \emph{term} is either a \emph{variable} or a \emph{constant}, where \emph{variables} are alphanumeric strings starting with uppercase letter, while \emph{constants} are either integer number or alphanumeric string starting with lowercase letter. An \emph{atom} is an expression of the form $p(t_1,\ldots,t_n)$ where $p$ is a predicate of ariety $n$ and $t_1,\ldots,t_n$ are terms; it is \emph{ground} if all its terms are constants. We say that an atom $p(t_1,\dots,t_k)$ has \textit{signature} $p/k$. An atom $\alpha$ \textit{matches} a signature 
$p/k$ if $\alpha = p(t_1,\dots,t_k)$.
A \emph{literal} is either an atom $a$ or its negation $\naf a$, where $\naf$ denotes the negation as failure. A literal is said to be \emph{negative} if it is of the form $\naf a$, otherwise it is positive. For a literal $l$, $\overline{l}$ denotes the complement of $l$. More precisely, $\overline{l}=a$ if $l = \naf a$, otherwise $\overline{l}=\naf a$.
A \emph{normal rule} is an expression of the form $h \leftarrow b_1,\ldots,b_n$ where $h$ is an atom referred to as \emph{head}, denoted by $H_r$, that can also be omitted, $n\geq 0$, and $b_1,\ldots,b_n$ is a conjunction of literals referred to as \emph{body}, denoted by $B_r$. 
In particular a normal rule is said to be a \emph{constraint} if its head is omitted, while it is said to be a \emph{fact} if $n=0$.
A normal rule $r$ is \emph{safe} if each variable $r$ appears at least in one positive literal in the body of $r$. 
A \emph{program} is a finite set of safe normal rules.
In what follows we will use also choice rules, which abbreviate complex expressions~\cite{DBLP:journals/tplp/CalimeriFGIKKLM20}. 
A \emph{choice element} is of the form $h:l_1,\ldots,l_k$, where $h$ is an atom, and $l_1,\ldots,l_k$ is a conjunction of literals. A \emph{choice rule} is an expression of the form $\{e_1;\ldots;e_m\}\leftarrow b_1,\ldots,b_n$, which is a shorthand for the set of normal rules $h_i \leftarrow l_1^i,\ldots,l_{k_i}^i, b_1,\ldots,b_n,\naf nh_i$; $nh_i \leftarrow l_1^i,\ldots,l_{k_i}^i, b_1,\ldots,b_n,\naf h_i$, for each $i \in {1,\ldots,m}$ where $e_i$ are of the form $h_i:l_1^i,\ldots,l_{k_i}^i$ and $nh_i$ is a fresh atom not appearing anywhere else.

Given a program $P$, the \emph{Herbrand Universe} of $P$, $\mathcal{U}_P$, denotes the set of constants that appear in $P$, while the \emph{Herbrand Base}, $\mathcal{B}_P$, denotes the set of ground atoms obtained from predicates in $P$ and constants in $\mathcal{U}_P$. Given a program $P$, and $r \in P$, $ground(r)$ denotes the set of ground instantiations of $r$ obtained by replacing variables in $r$ with constants in $\mathcal{U}_P$. Given a program $P$, $ground(P)$ denotes the union of ground instantiations of rules in $P$. An \emph{interpretation} $I \subseteq \mathcal{B}_P$ is a set of atoms. Given an interpretation $I$, a positive (resp. negative) literal $l$ is true w.r.t.\ $I$ if $l \in I$ (resp. $\overline{l} \notin I$); otherwise it is false. A conjunction of literal is true w.r.t $I$ if all its literals are true w.r.t.\ $I$. An interpretation $I$ is a \emph{model} of $P$ if for every rule $r \in ground(P)$, $H_r$ is true whenever $B_r$ is true. Given a program $P$ and an interpretation $I$, the (Gelfond-Lifschitz) reduct~\cite{DBLP:journals/ngc/GelfondL91}, denoted by $P^I$, is defined as the set of rules obtained from $ground(P)$ by deleting those rules whose body is false w.r.t $I$ and removing all negative literals that are true w.r.t.\ $I$ from the body of remaining rules.
Given a program $P$, and a model $I$, then $I$ is also a \emph{answer set} of $P$ if no such $I'\subseteq I$ exists such that $I'$ is a model of $P^I$. For a program $P$, let $AS(P)$ denotes the set of answer sets of $P$, then $P$ is said to be \emph{coherent} if $AS\neq \emptyset$, otherwise it is \emph{incoherent}.
Given an answer set $S$ and a signature $\sigma$, the \textit{projection of $S$ on $\sigma$}
is the set $S_{|\sigma} = \{\alpha \in S: \alpha \text{ matches } \sigma\}$.

\paragraph{MUSes and MSMs} %
Consider a program $P$ and a set of objective atoms $O \subseteq \mathcal{B}_P$.
For $S \subseteq O$, we denote by $\textsf{enforce}(P,O,S)$ the program obtained from $P$ by adding a choice rule over atoms in $O$ (i.e. $\{o_1;\ldots;o_n\}\leftarrow$) and a set of constraints of the form $\leftarrow \naf o$, for every $o \in S$.
Intuitively, $\textsf{enforce}(P,O,S)$ denotes an augmentation of the program $P$ in which the objective atoms can be arbitrarily choosen (i.e. either as true or false) but the atoms in $S$ are \emph{enforced} to be true.

An \emph{unsatisfiable subset} for $P$ w.r.t.\ the set of objective atoms $O$ is a set of atoms $U \subseteq O$ such that $\textsf{enforce}(P,O,U)$ is incoherent. $US(P,O)$ denotes the set of unsatisfiable subsets of $P$ w.r.t.\ $O$. An unsastisfiable subset $U \in US(P,O)$ is a \emph{minimal unsatisfiable subset} (\MUS) of $P$ w.r.t.\ $O$ iff for every $U' \subset U$, $U' \notin US(P,O)$.
Analogously, an answet set $M\in AS(P)$ is a \emph{minimal stable model}~(\MSM) of $P$ w.r.t.\ the set of objective atoms $O$ if there is no answer set $M'\in AS(P)$ with $(M'\cap O) \subset (M \cap O)$.

%% file: section/preliminaries_ltlf.tex
Linear Temporal Logic~(LTL) \cite{DBLP:conf/focs/Pnueli77} is an extension of 
propositional logic which allows to reason over infinite sequences of propositional 
interpretations or \emph{traces}. \LTLf \cite{DBLP:conf/ijcai/GiacomoV13} is
a variant of this logic that considers only finite traces. Let $\mathcal{A}$ be a 
finite set of propositional symbols. The class of \LTLf formulas over $\mathcal{A}$ 
is defined according to the grammar
\[
\varphi := a \mid \varphi \land \varphi \mid \lnot \varphi \mid \varphi\Until\varphi \mid \Next \varphi \mid \top
\]
where $a \in \mathcal{A}$. 
A formula is in \emph{conjunctive form} if it is expressed as a conjunction of
formulas. In this case, we often represent a formula as the set of its
conjuncts; i.e., the formula $\phi_1\land\cdots\land\phi_k$ is expressed
by the set $\{\phi_1,\ldots,\phi_k\}$.

A \emph{state} is any subset of $\mathcal{A}$; a \emph{trace} 
is a finite sequence $\pi = \sigma_0 \cdots \sigma_k$ of states; in this case,
the trace has length $|\pi| = k+1$.
The $i$-th state of the trace $\pi$ is denoted by $\pi(i) = \sigma_i$
The \emph{satisfaction relation} is defined recursively over the structure of 
$\varphi$. Let $\pi$ be a trace and $0 \le i < |\pi|$. 
We say that $\pi$ satisfies $\varphi$ at time $i$, denoted by $\pi,i\models\varphi$ 
iff:
\begin{itemize}
\item $\pi, i \models a \in \mathcal{A}$ iff $a \in \pi(i)$;
\item $\pi, i \models \top$;
\item $\pi, i \models \Next \varphi'$ iff $i<|\pi|-1$ and $\pi, i+1 \models \varphi'$;
\item $\pi, i \models \psi \land \phi$ iff $\pi,i \models \psi$ and $\pi,i \models \phi$; and
\item $\pi, i \models \psi \Until \phi$ if there exists $i \le j < |\pi|$ s.t.\ 
    $\pi,j\models \phi$ and for all $i \le k < j$, $\pi,k\models\psi$.
\end{itemize}
The trace $\pi$ is a \emph{model} of $\varphi$ (denoted by $\pi\models\varphi$) 
whenever $\pi, 0 \models \varphi$. 
The \emph{satisfiability problem} is the problem of deciding whether a 
formula $\varphi$ admits a model; i.e., if there exists $\pi$ such that 
$\pi\models\varphi$. \LTLf satisfiability is well-known to be PSpace-complete
\cite{DBLP:conf/ijcai/GiacomoV13}.

Given an unsatisfiable formula $\varphi=\{\phi_1,\ldots,\phi_k\}$ in conjunctive 
form, a \emph{minimal unsatisfiable core}~(MUC) of $\varphi$ is an unsatisfiable 
formula $\psi \subseteq \varphi$ which is minimal (w.r.t.\ set inclusion); i.e.,
removing any conjunct from $\psi$ yields a satisfiable formula \cite{10.1093/logcom/exad049}. 
Complexity-wise, it is known that a single formula may have 
exponentially many MUCs, but computing one MUC requires only polynomial
space; just as deciding satisfiability \cite{Pena19,Pena20}.

%% file: section/method.tex
\subsection{MUS and Probes}

In the rest of the paper we adopt the notation introduced in in~\cite{10.1093/logcom/exad049}. Let $\varphi = \{\phi_1, \dots, \phi_n\}$ be a formula in conjunctive form, where $\phi_i$ is a \textit{conjunct} of $\varphi$. With a slight abuse of notation, we will identify $\varphi$ with the \textit{set} of its conjuncts.

Our first assumption is that there exists an uniform way to encode \LTLf formulae in conjunctive normal form into logic programs. In particular, we are interested in \textit{encodings where original conjuncts of $\varphi$ can be told apart by means of special atoms}. More formally:

\begin{definition}[Reification Function]\label{def:reification} A \textit{reification function} for a formula $\varphi$ is a function that maps $\varphi$ into a logic program whose Herbrand base contains an atom $phi(i)$ for each $\phi_i \in \varphi$. We denote the set of atoms matching signature $phi/1$ by $\mathcal{O}(\varphi)$.
\end{definition}

Each subset $S \subseteq \mathcal{O}(\varphi)$ uniquely identifies the set of conjuncts $\psi = \{\phi_i: phi(i) \in \mathcal{O}(\varphi)\} \subseteq \varphi$. Therefore, we denote $\psi$ by $\Formula(S)$ and $S$ by $\ObjAtoms(\psi)$.

Reification functions enable to encode \LTLf formulae into logic programs. Since in this paper we are concerned with notions of \emph{satisfiability}, \emph{unsatisfiability} wrt \emph{subset minimality} of \LTLf formulae, among all possible reification functions, we are interested in ones that preserve as much information about these properties. In particular, we introduce the notion of \emph{probe}.

\begin{definition}[$k$-Probe]\label{def:probe} Let $k \in \mathbb{N}$. A reification function $\rho$ is a \emph{probe of depth $k$} (or \emph{$k$-probe} for short) for $\varphi$ if for each set $S \subseteq \mathcal{O}(\varphi)$ it holds that $\Formula(S)$ admits a model of length at most $k$ if and only if there exists an answer set $M$ of $\rho(\varphi)$ such that $S = M_{|\mathcal{O}(\varphi)}$.
\end{definition}

There exist multiple ASP encodings that satisfy the definition of probe. Intuitively, one can obtain a probe by adapting any ASP encoding for bounded \LTLf satisfiability~\cite{DBLP:conf/lpnmr/FiondaIR24,DBLP:journals/jair/FiondaG18}. Section~\ref{sec:theprobe} features an extended and detailed example. Here, we focus on how probes relate to MUCs of $\varphi$.

\begin{lemma}\label{lemma:musimpliesunsatorkbound}
Let $\rho$ be a probe of depth $k$ for $\varphi$. Let $S$ be a minimal unsatisfiable subset of $\rho(\varphi)$ wrt the objective atoms $\mathcal{O}(\varphi)$. Then $\Formula(S)$ is either an MUC of $\varphi$ or it is satisfiable but its shortest satisficing trace has length greater than $k$.
\end{lemma}
\begin{proof} Assume $S$ is a minimal unsatisfiable subset wrt $\mathcal{O}(\varphi)$. Then, all its (proper) subsets can be extended to answer sets --- thus, interpreting them as formulae yields an \LTLf formula that admits a model of length at most $k$, by Definition~\ref{def:probe}. Hence, all proper subsets of $\Formula(S)$ are satisfiable, while $\Formula(S)$ itself is either unsatisfiable or its shortest model trace has a length greater than $k$. In the former case, it matches the definition of MUC.
\end{proof}

We provide an example.
\begin{example}\label{example:shrinkmuc} Consider the formula $\varphi = \{\Next^5 \beta, \Next^5 \lnot \beta\}$. This formula has a unique MUC, namely $\Next^5 \beta \land \Next^5 \lnot \beta$. If we consider a probe $\rho_3 = \rho(3, \varphi)$, it has two MUSes, namely $\{\Next^5 \beta\}, \{\Next^5 \lnot \beta\}$, since these formulae \textit{do not admit models of length at most 3}. If we consider instead probes of depth at least 5, it is now possible to detect the MUC through the (unique) MUS $\{\Next^5 \beta, \Next^5 \lnot \beta\}$
\end{example}

Formulae exhibiting the property shown in the statement of Lemma~\ref{lemma:musimpliesunsatorkbound} are the key objects which allow us to levarage MUS enumeration to enumerate MUCs. Thus, we introduce a definition.

\begin{definition}[$k$-bound MUC]\label{def:kmuc}
Let $k \in \mathbb{N}$. A \emph{$k$-bound MUC} (or \emph{$k$-MUC}) for the formula $\varphi$ is a minimal subset of $\varphi$ 
that does not admit a model of length at most $k$. We denote by $\MUC^k(\varphi)$ the set of all $k$-MUCs for a formula $\varphi$. 
\end{definition}

\begin{lemma}
\label{lem:unsatkmucismuc}
Let $S \in \MUC^k(\varphi)$. If $S$ is unsatisfiable, then $S$ is a MUC for $\varphi$.
\end{lemma}

\begin{proof}
Follows from the fact that since $S \in \MUC^k(\varphi)$, it means that any proper subset of $S$ admits a model of length at most $k$, hence it satisfiable. If $S$ is also unsatisfiable, it matches the definition of MUC.
\end{proof}

We re-state Lemma~\ref{lemma:musimpliesunsatorkbound} adopting the new definition:

\begin{lemma}\label{lemma:musiskmuc}
Let $\varphi$ be a formula, $k \in \mathbb{N}$. $S$ is a minimal unsatisfiable subset of the $k$-probe $\rho(\varphi)$ if and only if $\Formula(S)$ is a $k$-MUC for $\varphi$.
\end{lemma}

Example~\ref{example:shrinkmuc} highlights an interesting property. The probe at depth $k = 3$ yields two (singleton) MUSes, that intepreted as formulae indeed do not admit models of length at most $k$. However, increasing the probe depth to $k' = 5$, yields a \textit{single MUS}, since the MUSes (of the previous probe) are actually both satisfiable if we consider models of length at most $k'$, but still (jointly) unsatisfiable considering models of length at most $k'$. Intuitively, this makes the probe at depth $k'$ \textit{more effective}, since it allows to \textit{discard} MUSes that won't lead to a MUC.

In this regard, with the aim of enumerating MUCs, the most interesting probes would be the ones that allow to detect all MUCs with no false positives. More formally, the \textit{most effective probe} is a probe at a depth $k^*$ such that for each $k' \ge k^*$ it holds if $\alpha$ is an MUS in a $k^*$-probe, it will also be an MUS in the $k'$-probe. We provide an argument to show that such a probe depth $k^*$ exists.

If $\varphi = \{\phi_1, \dots, \phi_n\}$, $\varphi$ can have at most $2^n$ MUCs. Let $h(\varphi)$ be the least integer $z \in \mathbb{N}$ such that \textit{any} satisfiable subset of $\varphi$ admits a model of length at most $z$. We refer to $h(\varphi)$ as the \textit{completeness threshold} for $\varphi$, and probes of depth greater or equal to $h(\varphi)$ as \textit{complete probes}.

This leads us to the following theorem, which establishes a bijection between MUSes of complete probes and MUCs of $\varphi$:

\begin{theorem}\label{thm:musequalskmuc} Let $\rho$ be a complete probe for $\varphi$. Then $S$ is an MUS of $P$ wrt $\mathcal{O}(\varphi)$ if and only if $\Formula(S)$ is a MUC of $\varphi$.
\end{theorem}
\begin{proof} $(\rightarrow)$ Let $S$ be an MUS of $\rho$ with respect to $\mathcal{O}(\varphi)$. By Lemma~\ref{lemma:musiskmuc} $\Formula(S)$ is a $k$-MUC, thus it is either unsatisfiable (hence it is a MUC); or satisfiable with a satisficing trace with length greater than $k$ --- in this latter case, $\rho$ would not be a complete probe. Hence, $\Formula(S)$ must be a MUC for $\varphi$.
$(\leftarrow)$ Let $\psi$ be an MUC of $\varphi = \{\phi_1, \dots, \phi_n\}$. Without loss of generality, we can assume $\psi = \{\phi_1, \dots, \phi_m\}$. All its proper subsets $\psi^j$ --- which denotes the \LTLf formula obtained by removing from $\psi$ the $j$-th conjunct --- are satisfiable. Since $\rho$ is a complete probe, for each $\psi^j$ there exists an answer set of $\rho(\varphi)$ that extends $\ObjAtoms(\psi^j)$, but there exists no answer set that extends $\ObjAtoms(\psi)$. Since $\ObjAtoms(\psi^j) = \ObjAtoms(\psi) \setminus \{phi(j)\}$, this shows that $\ObjAtoms(\psi)$ is an MUS for $\rho(\varphi)$ wrt $\mathcal{O}(\varphi)$. 
\end{proof}
Theorem~\ref{thm:musequalskmuc} characterizes MUCs of $\varphi$ as MUSes of complete probes for $\varphi$. From the standard automata-based procedure for deciding satisfiability in \LTLf
\cite{DBLP:conf/ijcai/GiacomoV13,MaMP20} it follows that every satisfiable 
formula $\varphi$ has a model of length at most $2^n$ where $n$ is the number of 
subformulas of $\varphi$.
Indeed, completeness threshold is bounded above by the upper bound on model length of $\varphi$ (due to monotonicity of \LTLf wrt conjunction --- adding a conjunct can only \textit{increase} the length of the shortest model). However, in practice, it can be much smaller, as we will see in the experiments section.
\subsection{MUC enumeration by MUS enumeration}
Applying Theorem~\ref{thm:musequalskmuc} we can enumerate MUCs of $\varphi$ by enumerating MUSes of a complete probe for $\varphi$. In general, computing the completeness threshold for $\varphi$ is not feasible. However, by Lemma~\ref{lem:unsatkmucismuc}, we also know that \textit{some} $k$-MUCs, with $k \le h(\varphi)$ could also be MUCs. These results suggest two anytime algorithms that could be useful in the realm of \LTLf MUC enumeration: \textit{(i)} an algorithm (cfr. Algorithm~\ref{alg:algorthm_kapprox}) that computes all MUCs among the $k$-MUCs for a given $k$ and \textit{(ii)} an iterative deepening variant of Algorithm~\ref{alg:algorthm_kapprox} (cfr. Algorithm~\ref{alg:algorithm_deepen}) which expands the probe depth $k$ whenever a $k$-MUC reveals not to be unsatisfiable.
Algorithm~\ref{alg:algorthm_kapprox} and \ref{alg:algorithm_deepen} provide pseudo-code for such approaches. Both algorithms make use of the subroutines \ttt{probe}, \ttt{enumerate\_mus}, \ttt{to\_formula}, \ttt{check\_satisfiability}, that are explained next.
\begin{description}
\item[\texttt{probe}$(\varphi,k)$] builds the logic program from which we will extract $k$-MUCs. This is the counterpart of $\rho(k, \varphi)$.
\item[\texttt{enumerate\_mus}($P$)] invokes an ASP solver to extract MUSes of the probe $P$ wrt the objective atoms $\Phi$;
\item[\texttt{to\_formula}$(x)$] given an MUS $x$ of $P$, rebuilds the \LTLf formula $\Formula(x)$;
\item[\texttt{check\_satisfiability($\psi$)}] determines wheter an \LTLf formula is satisfiable or not; if $\psi$ is satisfiable, returns the length of a satisficing trace; otherwise, it returns 0;
\end{description}

We remark both algorithms are compatible with any ASP solver that implements MUS enumeration (that is, an implementation of the procedure \texttt{enumerate\_mus}) and (complete) \LTLf solvers that can \textit{(i)} provide a satisficing trace length for satisfiable formulae \textit{(ii)} prove unsatisfiability (that is, an implementation of the procedure \texttt{check\_satisfiability)}.

\begin{algorithm}[tb]
\caption{Enumerate unsatisfiable $k$-MUCs}
\label{alg:algorthm_kapprox}
\begin{algorithmic}[1]
\STATE \textbf{def} \texttt{enumerate\_k\_mucs}($\varphi$, $k$):
\STATE \hspace{1em} $mucs = []$
\STATE \hspace{1em} $P = \texttt{probe}(\varphi, k)$
\STATE \hspace{1em} \textbf{for} $x$ \textbf{in} \texttt{enumerate\_mus}(P):
\STATE \hspace{3em} $\psi = \texttt{to\_formula}(x)$
\STATE \hspace{3em} $k' = \texttt{check\_satisfiability}(\psi)$:
\STATE \hspace{3em} \textbf{if} $k' = 0$:
\STATE \hspace{4em} $mucs.\text{append}(\psi)$
\STATE \hspace{1em} \textbf{return} $mucs$

\end{algorithmic}
\end{algorithm}

\begin{algorithm}[tb]
\caption{Enumerate MUCs - Iterative Deepening}
\label{alg:algorithm_deepen}
\begin{algorithmic}[1]
\STATE \textbf{def} \texttt{enumerate\_mucs}($\varphi$):
\STATE \hspace{1em} $k = 1$
\STATE \hspace{1em} $complete = False$
\STATE \hspace{1em} $mucs = []$
\STATE \hspace{1em} \textbf{while} not $complete$:
\STATE \hspace{2em} $P = \texttt{probe}(\varphi, k)$
\STATE \hspace{2em} $complete = True$
\STATE \hspace{2em} \textbf{for} $x$ \textbf{in} \texttt{enumerate\_mus}(P):
\STATE \hspace{3em} $\psi = \texttt{to\_formula}(x)$
\STATE \hspace{3em} \textbf{if} $x \in mucs$:
\STATE \hspace{4em} $skip$
\STATE \hspace{3em} $k' = \texttt{check\_satisfiability}(\psi)$:
\STATE \hspace{3em} \textbf{if} $k' = 0$:
\STATE \hspace{4em} $mucs.\text{append}(\psi)$
\STATE \hspace{3em} \textbf{else}:
\STATE \hspace{4em} $k = k'$
\STATE \hspace{4em} $complete = False$
\STATE \hspace{4em} \textbf{break}

\STATE \hspace{1em} \textbf{return} $mucs$

\end{algorithmic}
\end{algorithm}

Algorithm~\ref{alg:algorthm_kapprox} is straightforward. We enumerate MUSes of a $k$-probe, which yields a sequence of $k$-MUCs. Each $k$-MUC is a \textit{candidate} MUC for $\varphi$, that can be \textit{certified} or \textit{disproved} by a call to an \LTLf satisfiability oracle. Following such a call, we discard \textit{false positives} candidate (that is, $k$-MUCs that are actually satisfiable) as we meet them. This approach does not enable to detect all MUCs, unless $k \ge h(\varphi)$. Conjuncts whose shortest model has length greater than $k$ will be discarded.

Algorithm~\ref{alg:algorithm_deepen} extends Algorithm~\ref{alg:algorthm_kapprox}. Whenever we encounter a false positive $k$-MUC $\psi$, this is a witness of the fact the current $k$ is below the completeness threshold for $\varphi$. Thus, we increase $k$ according to the length of the model $\pi$ that satisfies $\psi$. Since $h(\varphi)$ is finite, $k$ will eventually converge to $h(\varphi)$. At that point, all $k$-MUCs of the $h(\varphi)$-probe result in MUCs for $\varphi$.

%% file: section/probe.tex
Probes can be obtained with slight modifications from any ASP encoding to perform bounded satisfiability of \LTLf formulae. In this section, we show how to obtain a probe from the encoding proposed by~\cite{DBLP:conf/lpnmr/FiondaIR24}, which repurposes to ASP the SAT-based approach presented in~\cite{DBLP:journals/jair/FiondaG18}. This will also be the probe we use in the experimental section. 
In rest of the section, we will provide ASP encoding using the clingo input language, for further detail we refer the reader to~\cite{DBLP:journals/tplp/GebserKKS19}.

We start by a brief recap of the ASP approach to bounded satisfiability~\cite{DBLP:conf/lpnmr/FiondaIR24}, then show how the encoding can be seamlessy adapted into a probe.

\paragraph{Encoding formulae.} The starting point is to encode an \LTLf formula $\varphi$ into a set of facts. Each subformula of $\varphi$ is assigned an unique integer identifier. This identifier is used as a term in the predicates $until/3$, $release/3$, $negate/2$, $conjunction/2$, $disjunction/2$ and $atom/1$ to reify the syntax tree of $\varphi$ into a directed acyclic graph.

\begin{example} As an example, consider the formula $\varphi = (a \land \lnot b) \land (c \Until b)$ with two conjuncts is
encoded through the facts:

\begin{verbatim}
conjunction(0, 1).  conjunction(0, 2).
conjunction(1, 3).  conjunction(1, 4).
atom(3, a).  negate(4, 5).  atom(5, b).  
until(2, 6, 5).  atom(6, c).
root(0).
\end{verbatim}

Additionally, the atom $root(i)$ encodes that $i$ is the root node of the formula $\varphi$. Without loss of generality, we can assume that the root node is always identified by 0. We denote by $[\varphi]$ the set of facts that encode the formula $\varphi$. 
\end{example}

\paragraph{Encoding \LTLf semantics.} 
The semantics of \LTLf temporal operators can be encoded into a recursive Datalog program. The logic program $\Pi_\text{semantics}$ below, described more in-depth in~\cite{DBLP:conf/lpnmr/FiondaIR24}, adapts the SAT-based approach described in~\cite{DBLP:journals/jair/FiondaG18}. %

\begin{verbatim}
holds(T,X) :- trace(T,A), atom(X,A).
holds(T,X) :- holds(T+1,F), 
  next(X,F), time(T+1).

holds(T,X) :- until(X,LHS,RHS),
  holds(T,RHS).

holds(T,X) :- holds(T,LHS), holds(T+1,X), 
  until(X,LHS,RHS).

holds(T,X) :- conjunction(X,_), time(T), 
  holds(T,F): conjunction(X,F).

holds(T,X) :- negate(X,F), 
  not holds(T,F), time(T)
\end{verbatim}

The predicate $trace/2$ is used to encode a trace. In particular, an atom $trace(t,a)$ models that $a \in \pi(t)$. We denote by $[\pi]$ the set of facts $\{trace(t,a): a \in \pi(t), 0\leq t < \vert \pi\vert \}$. The logic program $\Pi_\text{semantics} \cup [\pi] \cup [\varphi]$ admits a unique stable model $M$, such that $holds(0,0) \in M$ if and only if $\pi\models\varphi$.

\paragraph{Encoding \LTLf bounded satisfiabilty.} The $\Pi_\text{semantics}$ logic program can be used to evaluate whether $\pi\models\varphi$, where $\pi$ and $\varphi$ are suitably encoded into facts. This is straightforwardly adapted into a bounded satisfiability encoding $\Pi_{satisfiability}$, by replacing the set of facts encoding a specific trace with the following choice rule to $\Pi_\text{semantics}$: 

\begin{verbatim}
time(0..k-1).
{ trace(T,A): atom(_,A) } :- time(T).

:- root(X), not holds(X,0).
\end{verbatim}

In a typical \emph{guess \& check} approach, the above choice rule generate the search space of possible satisficing traces for $\varphi$ --- replacing a set  of facts $trace/2$ that encode a specific trace $\pi$. 

The constraint discards trace that are not models of $\varphi$. Thus, answer sets of $\Pi_{satisfiability} \cup [\varphi]$ are in one-to-one correspondance with traces of length at most $k$ that are models of $\varphi$. Note that $k$ is an input constant to the ASP grounder.

A more detailed account about the relationship between the original SAT encoding~\cite{DBLP:journals/jair/FiondaG18} and the ASP encoding is available in \cite{DBLP:conf/lpnmr/FiondaIR24}.

\paragraph{The probe.} The above program encodes whether $\varphi$ admits a model of length up to $k$. In order to comply with definition of probe (i.e. Definition~\ref{def:probe}), we require that there exists $\Pi_{satisfiability}$ admits an answer set \emph{for each subset of $\varphi$ that admits a model of length up to $k$}. This is obtained by replacing each fact of the form $conjunction(0,id) \in [\varphi]$, where $id$ is an identifier of a most immediate subformula of $\varphi$, with a rule of the form $conjunction(0,id) \leftarrow phi(id)$, as well as the choice rule $\{phi(id)\}\leftarrow$.

\begin{verbatim}
conjunction(1, 3).  conjunction(1, 4).
atom(3, a).  negate(4, 5).  atom(5, b).  
until(2, 6, 5).  atom(6, c).
root(0).

conjunction(0,1) :- phi(0). {phi(0)}.
conjunction(0,2) :- phi(1). {phi(1)}.
\end{verbatim}

As we can see, the only affected rules are the \texttt{conjunction} facts at the root level. Intuitively, the additional choice rule over $phi/1$ atoms \textit{enables} or \textit{disables} conjuncts of $\varphi$. This is reminiscent of how logic programs under stable model semantics are \textit{annotated} to exploit MUS enumeration for debugging purposes or to compute paraconsistent semantics~\cite{DBLP:journals/ai/AlvianoDFPR23}.

Finally, note that this is only an example, and probes could be obtained from different ASP encodings.
For example, a probe $P$ could encode the tableaux for 
\LTLf~\cite{DBLP:journals/corr/Reynolds16a,DBLP:journals/jar/GeattiGMV24}, or ad-hoc encodings for syntactical fragments of \LTLf such as Declare~\cite{DBLP:conf/padl/ChiarielloFIR24}.

%% file: section/experiments.tex
This section presents an experiment conducted to empirically evaluate the performance of our system \texttt{mus2muc}. 
We performed different experiments addressing the following issues:
\begin{itemize}
    \item [I] \textbf{Extraction of Single MUC}: \textit{How does \texttt{mus2muc} perform in computing a single MUC?}
    \item [II] \textbf{Enumeration of MUCs}: \textit{How effective is \texttt{mus2muc} in enumerating \LTLf MUCs?} 
    \item [III] \textbf{Generation vs. Certification}: \textit{How does MUSes generation and \LTLf satisfiability checks affect the overall performance of \texttt{mus2muc}?}
    \item [IV] \textbf{Domain agnostic MUCs enumeration techniques}: \textit{How does \texttt{mus2muc} compare with SAT-based MUCs enumeration techiques, suitably adapted from LTL to \LTLf domain?}
\end{itemize} 

In what follows, we describe the implementation of our system, and then shift the attention to an analysis aimed at answering the above questions.
\newcommand{\aaltafmuc}[2]{\texttt{aaltaf$\_$muc.#1#2}\xspace}
\newcommand{\aaltafuc}{\texttt{aaltaf$\_$uc}\xspace}
\newcommand{\musTomuc}{\texttt{mus2muc}\xspace}

\begin{figure*}[!t]
    \centering
    \begin{subfigure}{.49\textwidth}
        \includegraphics[width=\textwidth]{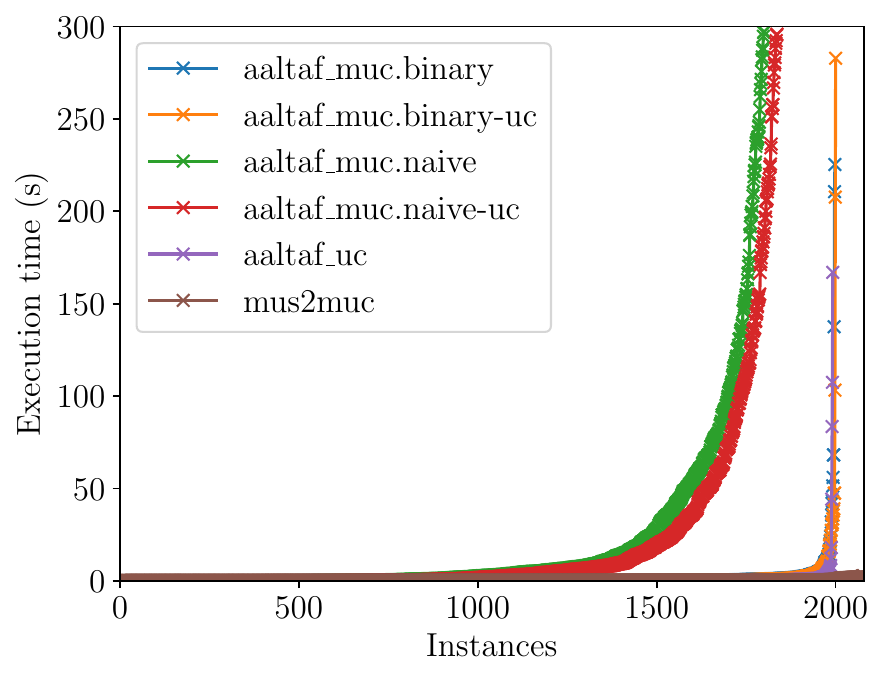}        
        \caption{All instances.}
        \label{fig:single-muc-full}
    \end{subfigure}
    \hfill 
    \begin{subfigure}{.49\textwidth}
        \includegraphics[width=\textwidth]{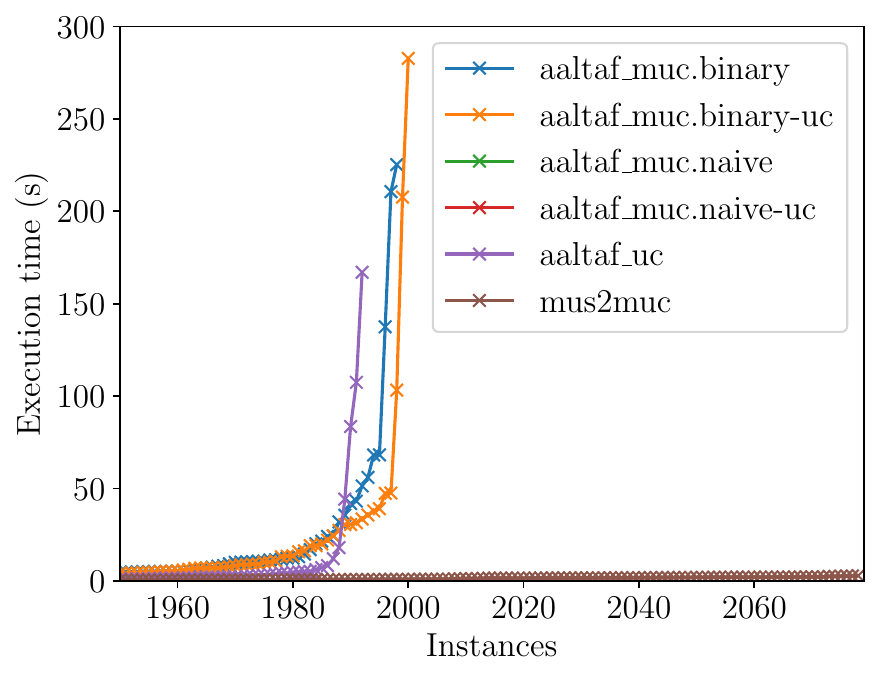}            
        \caption{Focus on the hardest instances (i.e., $x>1950$).}
        \label{fig:single-muc-cropped}
    \end{subfigure}
    \caption{Computation of a single MUC or UC.}
    \label{fig:single-mus-cactus}
\end{figure*}
\begin{table*}[!t]
\centering
    \resizebox{\textwidth}{!}{
    \begin{tabular}{@{}lrrrrrrrrrrrrrrr@{}}
    \multirow{2}{*}{\textbf{Benchmark}} && \multirow{2}{*}{\textbf{\#Inst.}} && \multirow{2}{*}{\textbf{\#Compl.}} && \multicolumn{2}{c}{\textbf{Sum of MUCs}} && \multicolumn{3}{c}{\textbf{Probe depth}} && \multicolumn{3}{c}{\textbf{MUC Size}}\\
    \cmidrule{7-8}\cmidrule{10-12}\cmidrule{14-16}
     & & & & & & \multicolumn{1}{c}{\textbf{Compl.}} & \multicolumn{1}{c}{\textbf{TO}} &&  {\textbf{Min.}} & {\textbf{Med.}} & {\textbf{Max.}} && {\textbf{Min.}} & {\textbf{Med.}} & {\textbf{Max.}}\\

    \midrule
    acacia.demo-v3                  && 11   && 11   && 77       & -         && 1     & 1   & 1   && 2  & 2   & 2       \\
    alaska.lift                     && 129  && 129  && 8310     & -         && 1     & 1   & 5   && 1  & 3   & 5       \\
    forobots.forobots               && 38   && 38   && 38       & -         && 1     & 1   & 2   && 2  & 2   & 2       \\
    schuppan.O1formula              && 27   && 27   && 27       & -         && 2     & 2   & 2   && 2  & 2   & 2       \\
    schuppan.O2formula              && 27   && 27   && 27       & -         && 1     & 1   & 1   && 2  & 60  & 1000    \\
    trp.N12x                        && 400  && 400  && 14380    & -         && 1     & 1   & 1   && 1  & 1   & 1       \\
    trp.N5x                         && 240  && 240  && 4210     & -         && 1     & 1   & 1   && 1  & 1   & 1       \\
    \midrule
    anzu.amba                       && 34   && 2    && 362      & 20642     && 5     & 9   & 9   && 1  & 8   & 41      \\
    anzu.genbuf                     && 36   && 6    && 1043     & 16119     && 6     & 6   & 6   && 1  & 4   & 9       \\
    rozier.counter                  && 76   && 62   && 514      & 35        && 4     & 23  & 41  && 2  & 4   & 5       \\
    schuppan.phltl                  && 13   && 9    && 219      & 2760      && 1     & 1   & 1   && 2  & 3   & 3       \\
    \midrule
    trp.N12y                        && 67   && 3    && 2514     & 116365    && 14    & 14  & 14  && 24 & 34  & 42      \\
    trp.N5y                         && 46   && 33   && 121704   & 287380    && 7     & 7   & 7   && 13 & 16  & 20      \\
    LTLfRandomConjunction.C100      && 500  && 134  && 15636    & 1295793   && 1     & 6   & 14  && 2  & 9   & 33      \\ 
    LTLfRandomConjunction.V20       && 435  && 152  && 64049    & 1619065   && 1     & 6   & 14  && 2  & 11  & 26      \\
    \midrule
    \bottomrule
    \end{tabular} 
    } 
    \caption{Complete MUC enumeration of the different formula families. \#Inst is the number of instances for each family. A benchmark $x.y$ denotes that the set of formulae $y$ is a \textit{family} of benchmark $x$.}
    \label{table:mus-enum}
\end{table*}
\begin{figure*}[t]
    \centering
    \begin{subfigure}{.23\textwidth}
        \includegraphics[width=\textwidth]{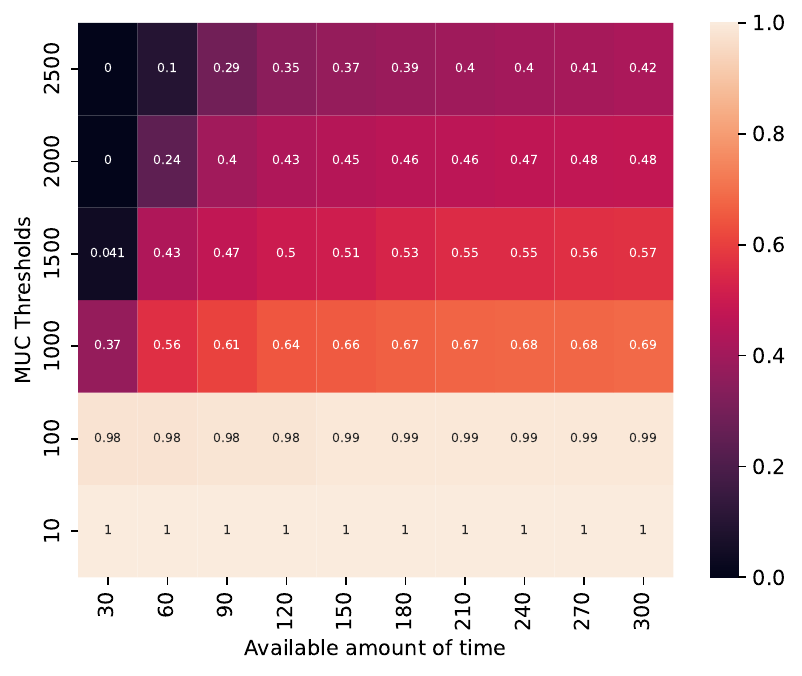}
        \caption{C100 RndConj class.\label{sub:heat_c100}}
    \end{subfigure}
    \begin{subfigure}{.23\textwidth}
        \includegraphics[width=\textwidth]{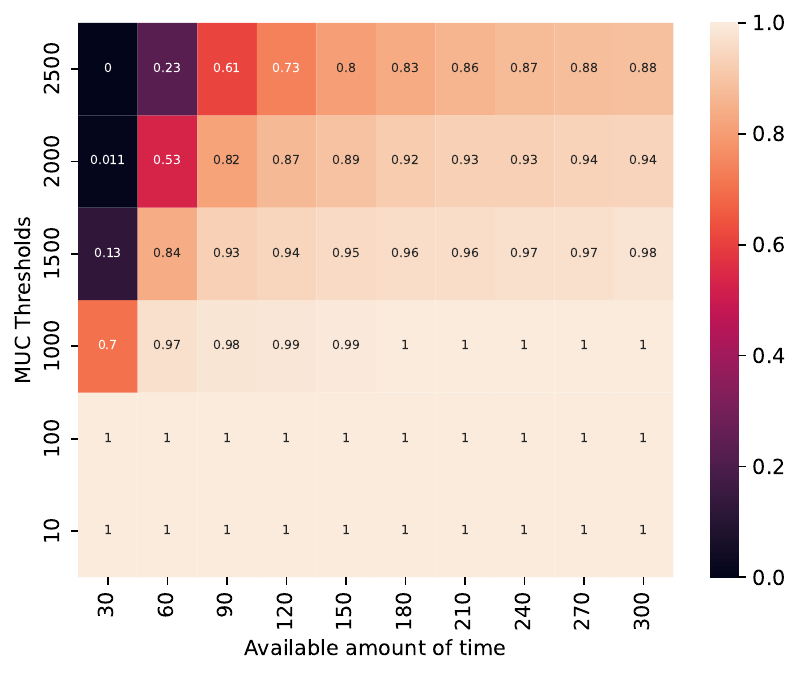}
        \caption{V20 RndConj class.\label{sub:heat_v20}}
    \end{subfigure}
    \begin{subfigure}{.23\textwidth}
        \includegraphics[width=\textwidth]{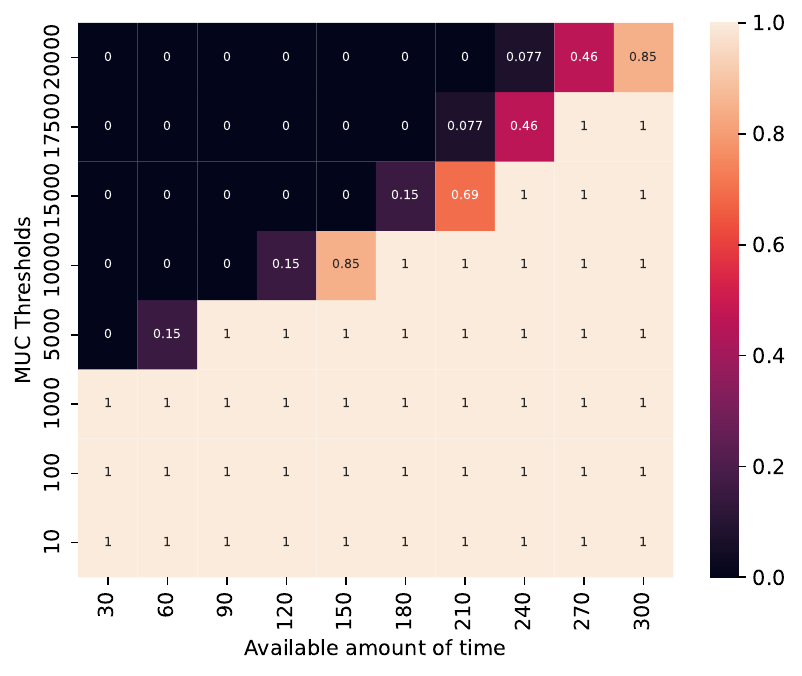}
        \caption{TRP5y class.\label{sub:heat_trp5}}
    \end{subfigure}
    \begin{subfigure}{.23\textwidth}
        \includegraphics[width=\textwidth]{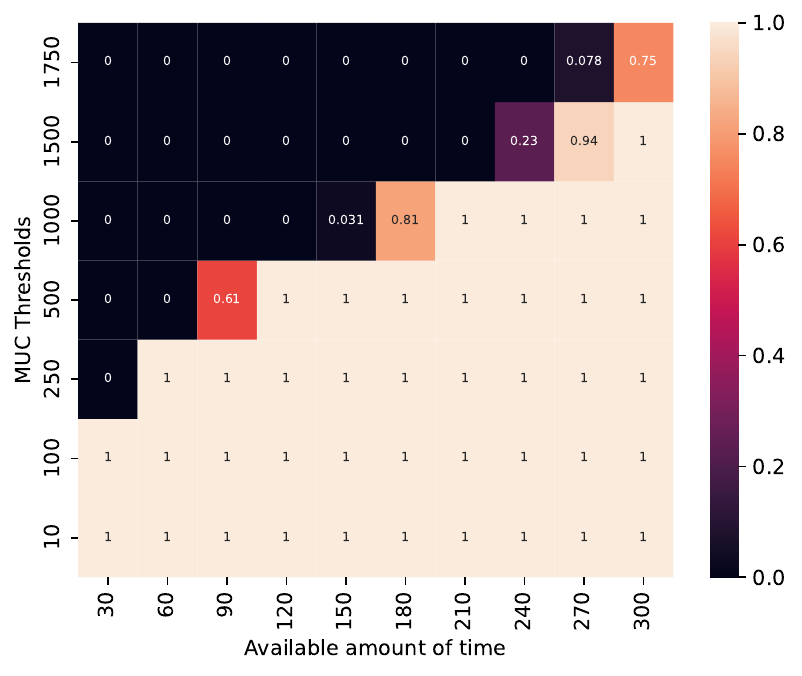}
        \caption{TRP12y class.\label{sub:heat_trp12}}
    \end{subfigure}
    \caption{Percentage of not-fully-enumerated instances (among the formula families RndConj C100, RndConj V20, TRP5y and TRP12y) that are able to enumerate at least a given number of MUCs in a certain amount of time. A cell $(i,j)$ in the heatmap represents the percentage of timed-out instances that are able to enumerate $j$ MUCs in $i$ seconds.}
    \label{fig:heatmaps_pct}
\end{figure*}

\begin{figure*}[t]
    \centering
    \begin{subfigure}{.23\textwidth}
        \includegraphics[width=\textwidth]{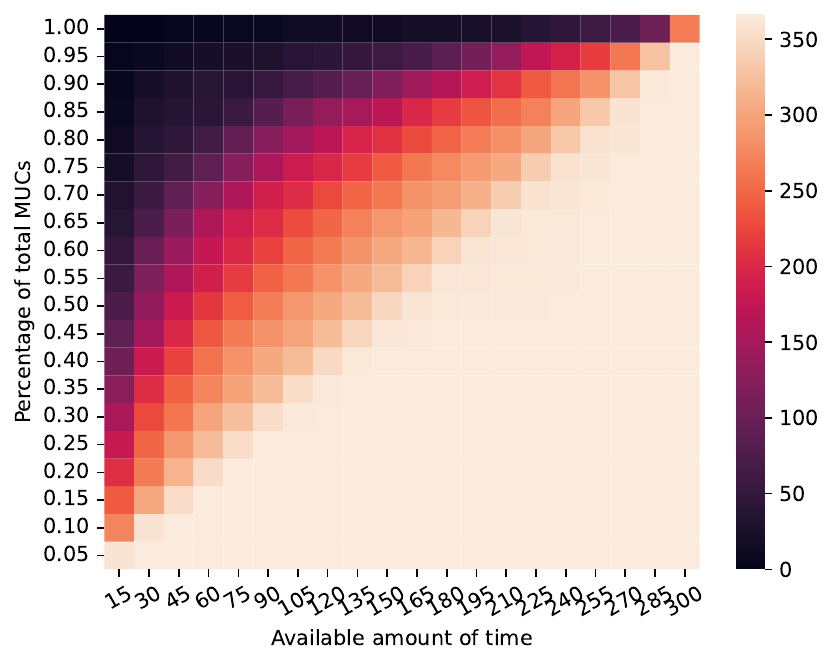}
        \caption{C100 RndConj class.\label{sub:heat_c100_pct}}
    \end{subfigure}
    \begin{subfigure}{.23\textwidth}
        \includegraphics[width=\textwidth]{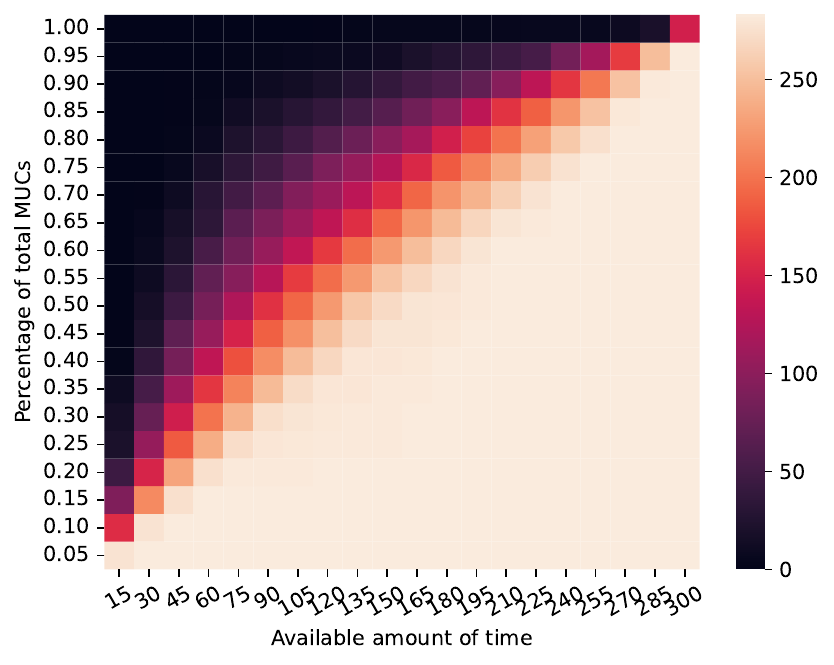}
        \caption{V20 RndConj class.\label{sub:heat_v20_pct}}
    \end{subfigure}
    \begin{subfigure}{.23\textwidth}
        \includegraphics[width=\textwidth]{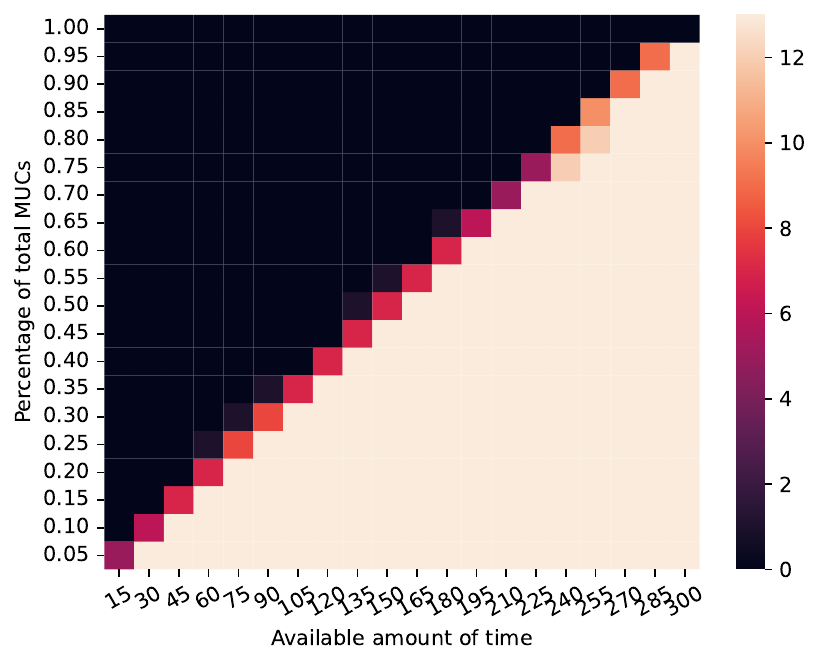}
        \caption{TRP5y class.\label{sub:heat_trp5_pct}}
    \end{subfigure}
    \begin{subfigure}{.23\textwidth}
        \includegraphics[width=\textwidth]{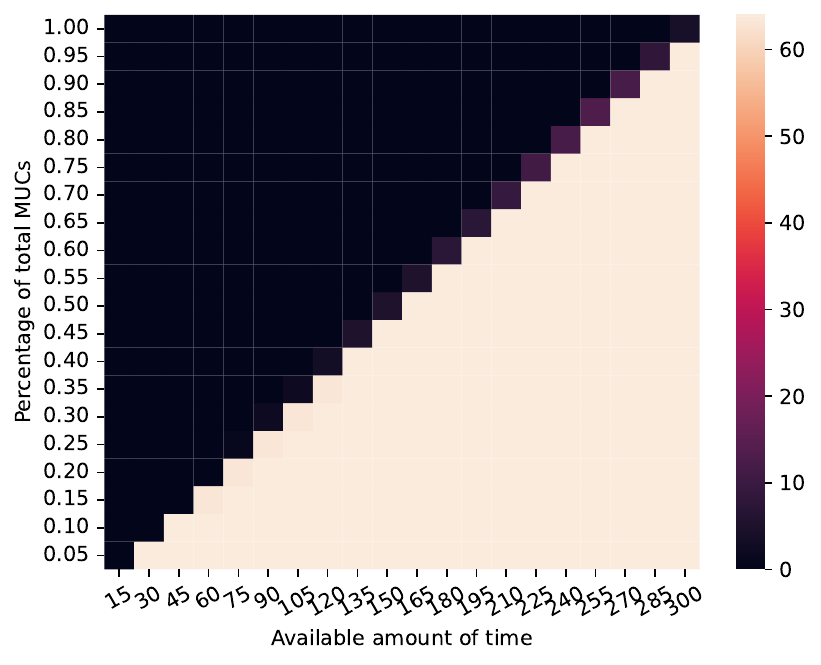}
        \caption{TRP12y class.\label{sub:heat_trp12_pct}}
    \end{subfigure}
    \caption{Number of not-fully-enumerated instances (among the formula families RndConj C100, RndConj V20, TRP5y and TRP12y) that are able to enumerate the $y$ percent of found MUCs within $x$ seconds of runtime. Thus, a cell $(i,j)$ in the heatmap represents how many instances in the given formula family can enumerate $i$ percent of the found MUCs found in 300s (up to timeout) in $j$ seconds.}
    \label{fig:heatmaps_num_instances}
\end{figure*}

\begin{figure*}[t]
    \centering
    \begin{subfigure}{.45\textwidth}
        \includegraphics[width=\textwidth]{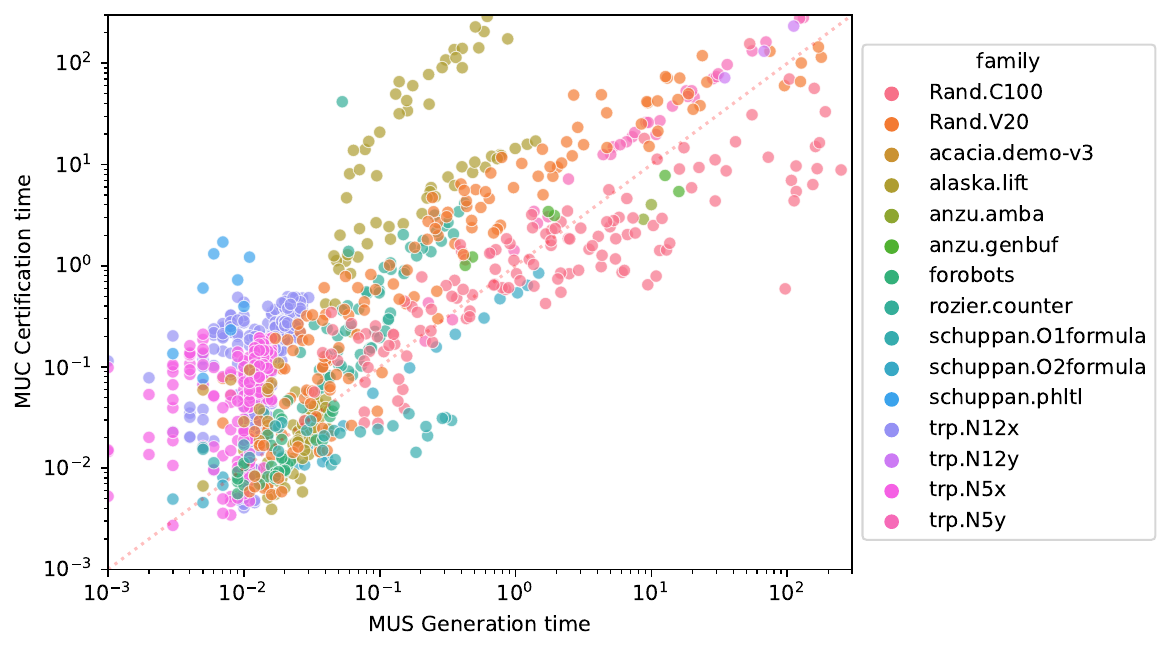}
        \caption{Fully-enumerated instances.\label{sub:time_gen_vs_time_cert_complete}}
    \end{subfigure}
    \begin{subfigure}{.45\textwidth}
        \includegraphics[width=\textwidth]{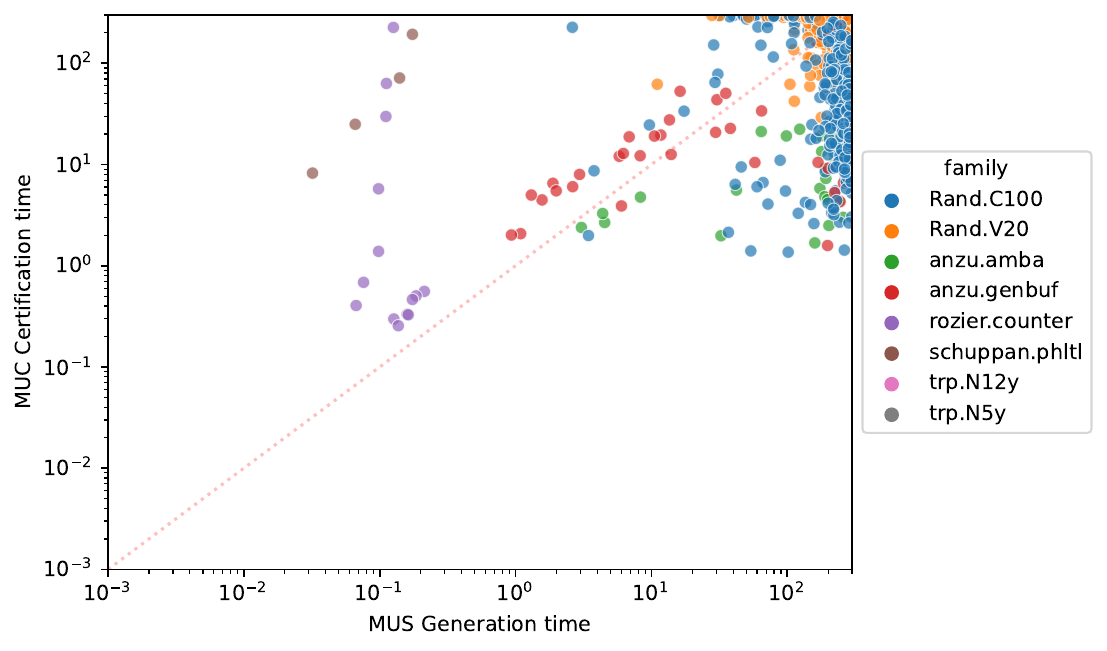}
        \caption{Incomplete instances.\label{sub:time_gen_vs_time_cert_incomplete}}
    \end{subfigure}
    \caption{A point $(x,y)$ in the scatter plot represents that a certain instance has spent $x$ CPU seconds generating $k$-MUCs (e.g., ASP MUS enumeration) and $y$ CPU seconds certifieing $k$-MUCs (e.g., an \LTLf solver running to prove unsatisfiability). Colors represent the formula family an instance belongs to. For a given point, $x + y$ might not sum up to the 300s timeout, because \textit{(i)} these modules run concurrently, thus CPU time can exceed 300s; \textit{(ii)} if the ASP solver or the \LTLf solver times out before yielding control, no event is recorded.}
    \label{fig:scatters}
\end{figure*}

\begin{figure*}[t]
    \centering
    \begin{subfigure}{.49\textwidth}
        \includegraphics[width=\textwidth]{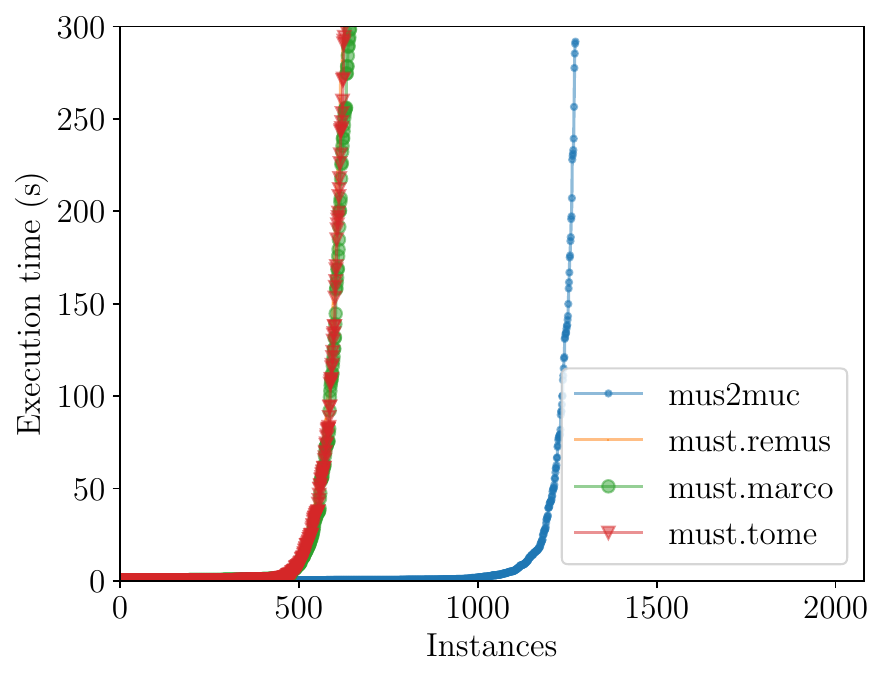}        
        \caption{Number of fully-enumerated instances.}
        \label{fig:enum-mucs-must-cactus}
    \end{subfigure}
    \hfill 
    \begin{subfigure}{.49\textwidth}
        \includegraphics[width=\textwidth]{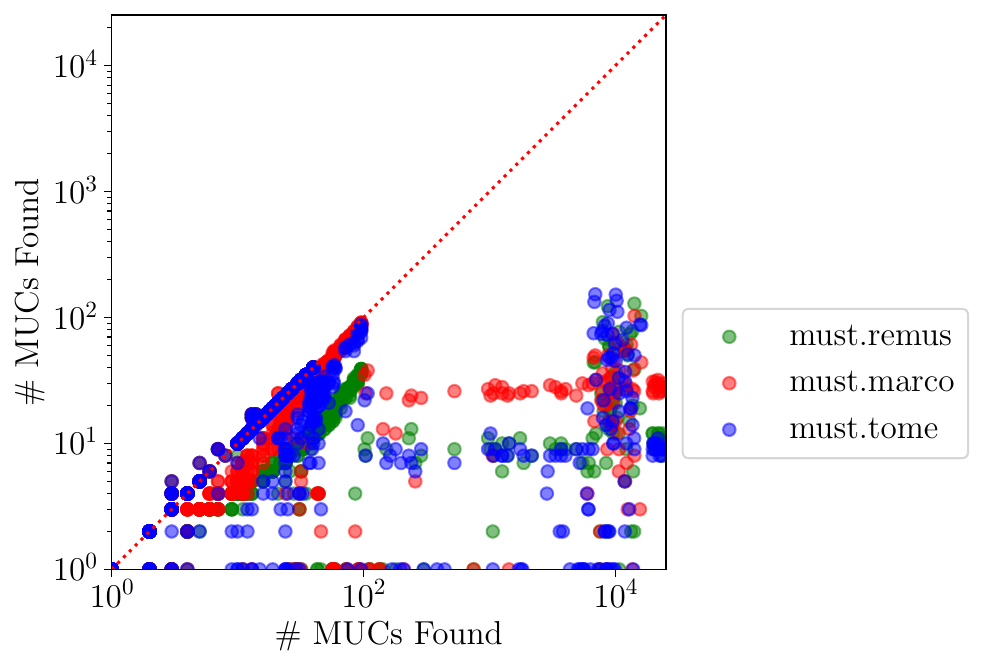}
        \caption{Number of MUCs enumerated within timeout for each instance.}
        \label{fig:enum-mucs-must-scatter}
    \end{subfigure}
    \caption{Enumeration of MUCs.}
    \label{fig:enum-mus-plots}
\end{figure*}
\paragraph{Implementation.} The implementation of \texttt{mus2muc} closely follows the pseudo-code in Algorithm~\ref{alg:algorithm_deepen}. In particular, our implementation uses the ASP solver \texttt{wasp} as a MUS Generator, and the \LTLf solver \texttt{aaltaf} as a satisfiability solver.
More in detail, the solver \texttt{wasp} takes as input the probe described in the previous section, and then performs the MUS enumeration.
As soon as a candidate $k$-MUC (i.e., a MUS of the probe) becomes available an instance of the \LTLf solver is invoked as a certifier, in a typical producer-consumer architecture. 
Furthermore, since multiple $k$-probes (for increasing value of $k$) are used, it is possible for $k$-MUCs to be produced multiple times (for different values of $k$). To avoid redundant calls to the \LTLf solver, we adopt a caching strategy on the MUS generator side. As stated in the previous section, our system is anytime, and outputs MUCs as soon as they are certified at the smallest $k$ that allows to do so.
Our implementation\footnote{Code will be made available upon request to the authors.} uses Python 3.12, the version of \texttt{aaltaf}\footnote{https://github.com/lijwen2748/aaltaf; \texttt{858885b}} and \texttt{wasp}\footnote{https://github.com/alviano/wasp; \texttt{f3e4c56}. Logging facilities for \texttt{mus2muc} require a patch available in our repository.} available on authors' public repositories.%

\paragraph{Systems.} For the single MUC extraction task, we compare with the \texttt{aaltaf-muc}\footnote{https://github.com/nuutong/aaltaf-muc; \texttt{
9b40837}} system, which computes a single minimal unsatisfiable core, in four configurations as described in~\cite{10.1093/logcom/exad049}. We also include \texttt{aaltaf-uc}\footnote{https://github.com/roveri-marco/aaltaf-uc; \texttt{b6aeb5c}}, which computes a single unsatisfiable core (with no minimality guarantees), and \texttt{black}\footnote{https://github.com/black-sat/black; \texttt{35cb36f}}, which implements a linear elimination strategy to extract a minimal unsatisfiable core. An in-depth comparison between \texttt{aaltaf-uc} and \texttt{aaltaf-muc} is available in~\cite{10.1093/logcom/exad049}.
For the MUCs enumeration task we consider a general purpose tool \texttt{must}~\cite{DBLP:conf/lpar/BendikC18}\footnote{https://github.com/jar-ben/mustool; \texttt{17fa9f9}}, which supports three LTL MUC enumeration algorithms (namely, \texttt{ReMUS}, \texttt{MARCO} and \texttt{TOME}). 
Note that, since \texttt{must} supports the LTL domain but not the \LTLf domain, we patch \texttt{must} by applying the well-known LTL-to-\LTLf transformation presented in \cite{DBLP:conf/ijcai/GiacomoV13} before evaluating \LTLf constraints within the MUC enumeration procedure.
For further details, we refer the reader to~\cite{DBLP:conf/lpar/BendikC18}. 

\paragraph{Benchmarks.}
In our experiments we consider a benchmark suite consisting of common formulae families used in LTL and \LTLf literature to evaluate solvers. In particular, we use all unsatisfiable formulae that appear in ~\cite{schuppan}, and randomly generated formulae from~\cite{vardi}.
These formulae have been previously used by \cite{10.1093/logcom/exad049} to benchmark single MUC computation, and by ~\cite{DBLP:journals/jair/RoveriCFG24} for single UC (with no minimality guarantee) extraction.
This benchmark suite contains unsatisfiable instances from 15 different applications domains, each with different formula shapes and feature. In particular, they comprise both instances from applications (13 domains) and randombly generated (2 domains).
In total, the suite contains 2079 unsatisfiable instances, that can be obtained from~\cite{schuppan}.
All instances were mapped in conjunctive form by recursively traversing the formula parse tree in a top-down fashion, stopping whenever formulae are not conjunctions. This is consistent with how these instances have been handled by ~\cite{10.1093/logcom/exad049,DBLP:journals/jair/RoveriCFG24}. All formulae are interpreted as \LTLf formulae.

\paragraph{Execution environment.}
The experiments were run on a system with 2.30GHz Intel(R) Xeon(R) Gold 5118 CPU and 512GB of RAM with Ubuntu 20.04.2 LTS (GNU/Linux 5.4.0-137-generic x86\_64).
For all experiments and systems, over each instance in the benchmark, memory and time were limited to 8GB and 300s of real time, 700s of CPU time respectively.

\paragraph{Extraction of a single MUC.}

First of all we assess the performance of our implementation in the computation of a single MUC. \cite{10.1093/logcom/exad049} provides two SAT-based approaches for single MUC extraction, called \emph{NaiveMUC} and \emph{BinaryMUC}, as well as two heuristic variants called \emph{NaiveMUC+UC} and \emph{BinaryMUC+UC} which augment the approach with techniques used in boolean unsatisfiable cores extraction. We refer the reader to \cite{10.1093/logcom/exad049} for an in-depth analysis of the techniques.

A related subtask is that of single unsatisfiable core extraction (UC), that is a set of unsatisfiable formulae with no subset-minimality guarantee. Algorithms for single UC extraction have been recently surveyed in \cite{DBLP:journals/jair/RoveriCFG24}, and \cite{10.1093/logcom/exad049} features a comparison between single MUC extraction introduced in \cite{10.1093/logcom/exad049} and techniques surveyed in \cite{DBLP:journals/jair/RoveriCFG24}.

In this experiment, we consider all algorithms for single MUC extraction featured in \cite{10.1093/logcom/exad049} (namely, \texttt{aaltaf-muc.binary}, \texttt{aaltaf-muc.naive}, \texttt{aaltaf-muc.binary-uc}, \texttt{aaltaf-muc.naive-uc} and the best-performing algorithm for single UC extraction features in \cite{DBLP:journals/jair/RoveriCFG24} (namely, \texttt{aaltaf-uc}). We compare with our system \texttt{mus2muc} configured to stop at the first MUC extracted from each \LTLf formula.

The cactus plot in Figure~\ref{fig:single-muc-full} reports the performance of different systems in this task. Overall, we can observe that most of the formulae are trivial for all systems, resulting in sub-second runtimes. Figure ~\ref{fig:single-muc-cropped} ``zooms-in'' to the hardest instances, were we observe that the \texttt{aaltaf-muc.binary} and \texttt{aaltaf-muc.binary-uc} are faster than \texttt{aaltaf-uc}, although the task solved by \texttt{aaltaf-uc} is easier (since it does not provide minimality guarantees on the UC). These results match the experimental results in \cite{10.1093/logcom/exad049}. Overall, \texttt{mus2muc} outperforms all systems in this task.

\paragraph{Enumeration of MUCs.}
Our second experiment consists in evaluating \texttt{mus2muc} effectiveness in \textit{enumerating} MUCs of the formulae in the benchmark suite. 
Table~\ref{table:mus-enum} reports statistics about the number of found MUCs, probe depth and size of MUCs (i.e., number of conjuncts). 

In general, different formula families exhibit heterogeneous behavior, ranging from \textit{easy} (e.g., fully enumerated within seconds) to \textit{hard} --- yielding a number of MUCs in the order of thousands per instance, that cannot be fully enumerated within the timeout. 
In particular, some of the \textit{easy} families can be fully-enumerated with a probe depth that does not exceed one. 
Essentially, all inconsistencies can be detected at a propositional level, involving no temporal reasoning.

For the remaining formula families, we study \textit{how fast} MUCs are extracted using \texttt{mus2muc}. The heatmaps in figures~\ref{sub:heat_c100}-~\ref{sub:heat_trp12} report, for distinct families, in a cell $(x,y)$ the percentage of instances for which \texttt{mus2muc} can produce at least $y$ MUCs in at most $x$ seconds. Even on these formulae, our approach is able to output a considerable amount of MUCs in short time, albeit not able to fully enumerate them within the timeout. 
Conversely, the heatmaps in figures~\ref{sub:heat_c100_pct}-~\ref{sub:heat_trp12_pct} report how MUCs are ``temporally distributed'' within the timeout. For distinct families, a cell $(x, y)$ contains the number of instances where it is possible to find $y$ percent of found MUCs (i.e., enumerated within timeout) within $x$ seconds. We can see that for all these families, in the majority of instances a MUCs are computed in a steady fashion and MUCs become available within seconds of runtime. 
Instances in these families are characterized by a huge number of MUCs that cannot be realistically inspected. However, even if in this scenario, our approach can provide a reasonable number of MUCs within few seconds.  

\paragraph{Generation vs. Certification.} In the \texttt{mus2muc} system, following Algorithm~\ref{alg:algorithm_deepen}, each (unique) MUS extracted from the probe is checked for satisfiability by an \LTLf solver, to be either \textit{certified} (e.g., found unsatisfiable) or \textit{disproved} (e.g., there exists a satisficing trace whose length exceeds the current probe depth). In our implementation, MUS search and MUS certification run concurrently rather than in an interleaved fashion. Given the modularity of our approach, it is interesting to study which component affects runtimes the most. To this end, we consider formula families that are not fully enumerated within timeout, but behave differently from the ones considered in the previous experiment. 

In particular, when performing MUC extraction over an instance $F$, a certain amount of seconds due to MUS generation and MUS certification are accrued. Scatter plots in Figure~\ref{fig:scatters} report each instance as a point $(x,y)$, where $x$ is the total CPU time spent running MUS generation procedures and $y$ is the total CPU time spent running MUS certification\footnote{Notice that, for a given instance, MUS generation runtimes and MUS certification runtimes do not necessarily sum up to the timeout since components run concurrently (i.e., CPU time could be greater than wall time). 
Furthermore, if a timeout signal is received \textit{while} a MUC is being certified, \texttt{aaltaf} can't output any timestamp. Same goes for \texttt{wasp} during MUS generation. This explains not-fully-enumerated instances below the upper-right corner of the scatter plot.}. Colors denote which family each data-point belongs to.

We can observe in Figure~\ref{sub:time_gen_vs_time_cert_incomplete} that MUS generation and MUC certification can both become bottlenecks in \texttt{mus2muc}, for unsatisfiable instances. Notably, some formula families such as $rozier.counter$ feature unsatisfiable instances for which \texttt{wasp} is able to provide MUSes in less than a second, but whose certification time exceeds the allowed runtime. Similarly, in the $C100$ random conjunction family features instances for which the cumulative certification time is one order of magnitude smaller than MUS generation time. This sort of trade-off can be better analyzed by considering only fully enumerated instances in Figure~\ref{sub:time_gen_vs_time_cert_complete}, where it is possible to observe heterogeneous behavior among families, ranging from families that are trivial from both standpoints (lower left corner); hard from both standpoints (upper right corner); easy MUS generation-wise, but hard MUC-certification wise (upper left corner). No fully-enumerated instances are easy certification-wise and hard generation-wise --- as we have a mostly empty lower right corner in scatter plot.

\paragraph{Domain-agnostic MUCs enumeration techniques.} As far as we know, no publicly available systems work out of the box to enumerate MUCs of \LTLf formulae. However, a number of \textit{general purpose}, \textit{domain-agnostic} MUC extraction algorithms (which also support LTL as a domain) are available~\cite{DBLP:conf/lpar/BendikC18}. 
The survey by \cite{DBLP:journals/jair/RoveriCFG24}, does not compare with algorithms proposed in \cite{DBLP:conf/lpar/BendikC18}.

Figure~\ref{fig:enum-mucs-must-cactus} compares the number of fully-enumerated instances among different \texttt{must} algorithms and \texttt{mus2muc}. \texttt{mus2muc} is more effective, and can fully-enumerate more or less 500 more instances than any \texttt{must} variant. All \texttt{must} variants perform roughly the same.
Figure~\ref{fig:enum-mucs-must-scatter} compares the number of found MUCs of each \texttt{must} variant wrt \texttt{mus2muc}. A point $(x,y)$ in Figure~\ref{fig:enum-mucs-must-scatter}, denotes that for a given instance in the benchmarks suite \texttt{mus2muc} has computed $x$ MUCs whereas one of the \texttt{must} algorithms has computed $y$ MUCs. Each color distinguish a specific \textit{must} algorithm.

We can see, from the cactus plot in Figure~\ref{fig:enum-mucs-must-cactus}, that different algorithms of \texttt{must} are able to fully-enumerate (roughly) the same number of instances, indeed corresponding lines are mostly overlapped.
Overall, \texttt{mus2muc} is able to enumerate more MUCs than any of the \texttt{must} variants --- in some extreme cases, enumerating several order of magnitude more MUCs (see the points that lie on the $x$-axis).